\documentclass{article}

\usepackage{arxiv}

\usepackage[utf8]{inputenc} 
\usepackage[T1]{fontenc}    
\usepackage{hyperref}       
\usepackage{amsmath}
\usepackage{url}            
\usepackage{booktabs}       
\usepackage{amsfonts}       
\usepackage{nicefrac}       
\usepackage{microtype}      
\usepackage{cleveref}       
\usepackage{graphicx}
\usepackage{natbib}
\usepackage{multirow}
\usepackage{doi}
\usepackage{authblk}
\usepackage{wrapfig}
\usepackage{subcaption}
\usepackage{tabularx}
\usepackage{algorithm}
\usepackage[noend]{algpseudocode}
 \usepackage{algorithmicx}
 \algdef{SE}[DOWHILE]{Do}{doWhile}{\algorithmicdo}[1]{\algorithmicwhile\ #1}%
\algrenewcommand\algorithmicrequire{\textbf{Input:}}
\algrenewcommand\algorithmicensure{\textbf{Output:}}
\usepackage{enumitem}
\usepackage{amsfonts}
\usepackage{amsthm}
\newtheorem{theorem}{Theorem}
\newtheorem{corollary}[theorem]{Corollary}

\title{GradES: Significantly Faster Training in Transformers with \underline{Grad}ient-Based \underline{E}arly \underline{S}topping}

\author[1,$\dagger$, *]{Qifu Wen}
\author[1,$\dagger$]{Xi Zeng}
\author[1]{Zihan Zhou}
\author[2]{Shuaijun Liu}
\author[3,4]{Mehdi Hosseinzadeh}
\author[2]{Ningxin Su}
\author[1,5]{Reza Rawassizadeh}
\affil[1]{Department of Computer Science, Boston University Metropolitan College}
\affil[2]{Information Hub, The Hong Kong University of Science and Technology, Guangzhou}
\affil[3]{School of Engineering and Technology, Duy Tan University, Da Nang, Vietnam}
\affil[4]{Department of AI, School of Computer Science and Engineering, Galgotias University, Greater Noida, India}
\affil[5]{Center of Excellence in Precision Medicine and Digital Health, Department of Physiology, Chulalongkorn University, Thailand}
\affil[ ]{\textsuperscript{$\dagger$}These authors contributed equally to this work}
\affil[ ]{\textsuperscript{*}Corresponding author: qfwen@bu.edu, rezar@bu.edu, ningxinsu@hkust-gz.edu.cn}

\renewcommand{\shorttitle}{}
\hypersetup{
pdftitle={GradES: Significantly Faster Training in Transformers with
Gradient-Based Early Stopping},
pdfsubject={cs.LG, cs.CL},
pdfauthor={Qifu Wen, Xi Zeng, Zihan Zhou, Shuaijun Liu, Mehdi Hosseinzadeh, Ningxin Su, Reza Rawassizadeh},
pdfkeywords={Transformers, Early Stopping, Fine-tuning, Optimization, Training Efficiency},
}
\begin{document}
\maketitle

\begin{abstract}
Early stopping monitors global validation loss and halts all parameter updates simultaneously, which is computationally costly for large transformers due to the extended time required for validation inference. We propose \textit{GradES}, a novel gradient-based early stopping approach that operates within transformer components (attention projections and Feed-Forward layer matrices). We found that different components converge at varying rates during fine-tuning for both language and vision-language models. \textit{GradES} tracks the magnitude of gradient changes in backpropagation for these matrices during training. When a projection matrix's magnitude of gradient changes fall below a convergence threshold $\tau$, we exclude that projection matrix from further updates individually, eliminating costly validation passes while allowing slow converging matrices to continue learning. \textit{GradES} speeds up training time by 1.57--7.22$\times$ while simultaneously enhancing generalization through early prevention of overfitting, resulting in 1.2\% higher average accuracy in language tasks and 3.88\% on multimodal benchmarks.
\end{abstract}


\section{Introduction}
Large language models (LLMs) have remarkable capabilities across diverse tasks, but their training and deployment require substantial computational costs that scale with model size and inference frequency. Due to the high cost of training models with billions of parameters, any effort toward reducing the training cost improves the development of LLMs. This challenge extends to vision-language models (VLMs), where we can observe similar differences in visual and textual modalities.

Fine-tuning LLMs requires balancing computational efficiency against downstream task performance. As transformer architectures scale to billions of parameters, the computation becomes increasingly expensive \cite{dettmers2023qlora, malladi2023mezo}. While optimizing for memory and computational efficiency remains important \cite{rawassizadeh2025machine}, one key issue is often overlooked, i.e., the common practice of extensive fine-tuning assumes that more gradient updates always improve performance ~\cite{zhang2021bert}. Research has shown that training for excessive epochs leads to overfitting, where models overfit training data and fail to generalize ~\cite{mosbach2021stability}. 

Conventional early stopping, which is determined based on loss score, is computationally expensive for large language models, as each validation step requires full forward passes through all transformer layers for every sample in the validation set. This overhead scales linearly with both model size and validation set size, forcing practitioners to validate infrequently, typically every few thousand training steps \cite{tirumala2022memorization}, thereby creating a fundamental trade-off between computational cost and the risk of overfitting. In vision transformers(ViT), the same problem is present because ViT process both image and text inputs through separate transformer encoders. We detail this overhead in ~\ref{section:results}. 

Furthermore, our experiment shows that Transformer's components have different convergence patterns, as shown in Figure~\ref {fig:within_layer_gradient}. This difference in convergence is particularly pronounced in ViT, where vision transformers typically converge slower than their language counterparts. and the binary decision of classic early stopping fails to exploit the diverse convergence patterns, where attention and MLP components exhibit fundamentally different learning dynamics during the fine-tuning process \cite{yao2025theoretical}.

Through analysis of gradient dynamics across transformer architectures, we identified a critical inefficiency in current fine-tuning practices, i.e., varied convergence across the Transformer's components. By tracking magnitude of gradient changes for attention projections matrix: $\mathbf{W}_q$, $\mathbf{W}_k$, $\mathbf{W}_v$ that compute queries, keys, and values respectively, and $\mathbf{W}_o$ or output projection. As well as MLP network matrix $\mathbf{W}_{\text{up}}$ for dimension expansion and $\mathbf{W}_{\text{down}}$ for dimension reduction. We observe that some components reach convergence with magnitude of gradient changes below $10^{-3}$, while others maintain substantial magnitude of gradient changes throughout, as shown in Figure~\ref {fig:within_layer_gradient}. This pattern holds across both language models and ViT architectures. 

\begin{figure}[htbp]
    \centering
    \includegraphics[width=0.8\textwidth]{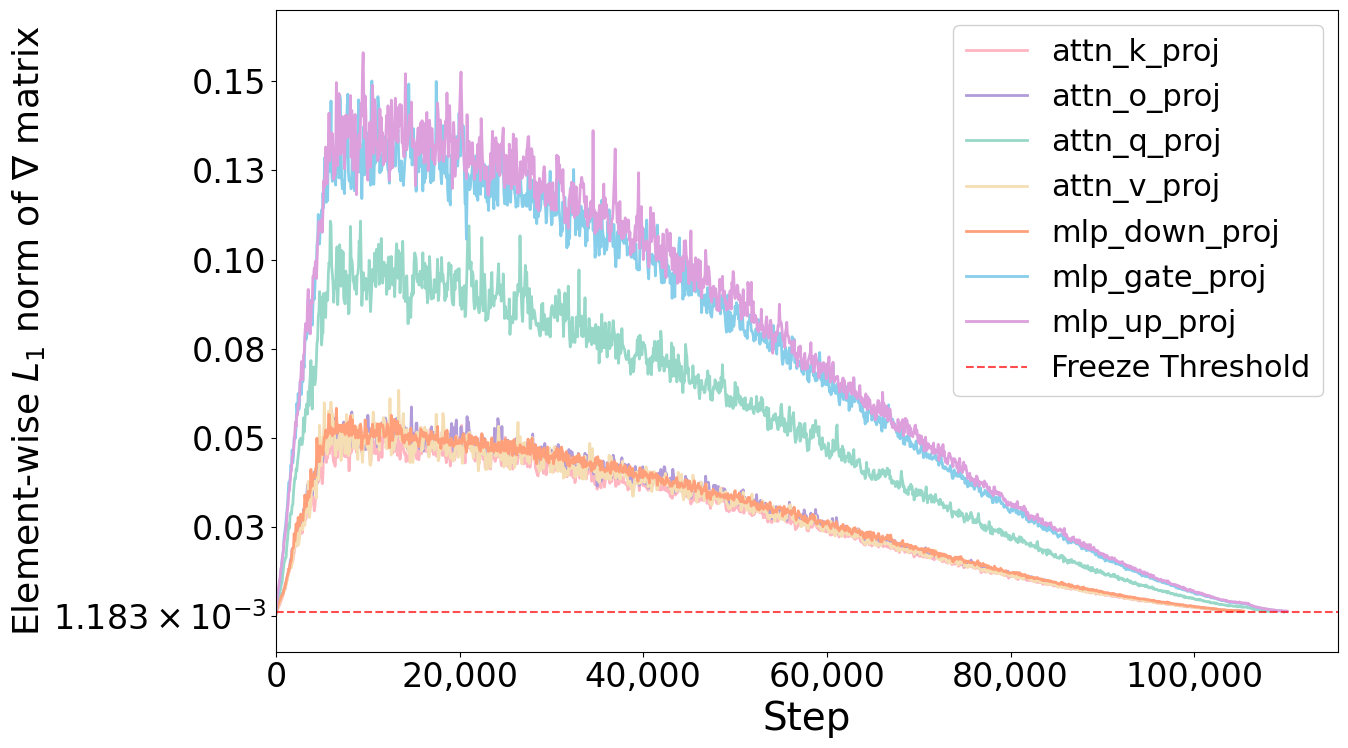}
    \caption{Element-wise $L_1$ norms for the gradient matrix of components in layer 7 for Qwen3-0.6B~\cite{qwen2024qwen25}. Each step consists of processing one training batch through the complete forward pass, loss computation, backpropagation, and parameter update cycle. The seven tracked matrices comprise attention projections ($\mathbf{W}_q^{(7)}$, $\mathbf{W}_k^{(7)}$, $\mathbf{W}_v^{(7)}$, $\mathbf{W}_o^{(7)}$) and MLP components ($\mathbf{W}_{\text{gate}}^{(7)}$, $\mathbf{W}_{\text{up}}^{(7)}$, $\mathbf{W}_{\text{down}}^{(7)}$). MLP projections exhibit 2 to 3$\times$ higher gradient magnitudes than attention projections throughout training, with $\mathbf{W}_{\text{up}}^{(7)}$ and $\mathbf{W}_{\text{down}}^{(7)}$ maintaining the largest magnitude of gradient changes. The red dotted line indicates our convergence threshold $\tau = 1.183 \times 10^{-3}$.} 
    \label{fig:within_layer_gradient}
\end{figure}

This disparity motivates \textit{GradES}, our gradient-based early stopping strategy that operates at the matrix level. Rather than requiring expensive validation passes, \textit{GradES} leverages gradient information already computed during backpropagation to monitor each matrix $\mathbf{W}^{(l)} \in \{\mathbf{W}_q^{(l)}, \mathbf{W}_k^{(l)}, \mathbf{W}_v^{(l)}, \mathbf{W}_o^{(l)}, \mathbf{W}_{\text{gate}}^{(l)}, \mathbf{W}_{\text{up}}^{(l)}, \mathbf{W}_{\text{down}}^{(l)}\}$ independently, where $l$ denotes the layer. As illustrated in Figure~\ref{fig:GradES_architecture}, our architecture monitors the L1 norm changes for each component within the transformer layers. When a matrix's magnitude of gradient changes falls below threshold $\tau$, the component is frozen (marked with lock icons in Figure~\ref{fig:GradES_architecture}), we stop its training while maintaining gradient flow for proper backpropagation, transforming early stopping from a binary termination decision into a continuous regularization mechanism. This selective stopping allows components with higher gradient activity to continue learning while converged components remain fixed. The selection of threshold $\tau$ is a critical hyperparameter whose configuration is detailed in Section~\ref{section:hyperparameter}.

\begin{figure}[htbp] 
\centering 
\includegraphics[width=0.8\textwidth]{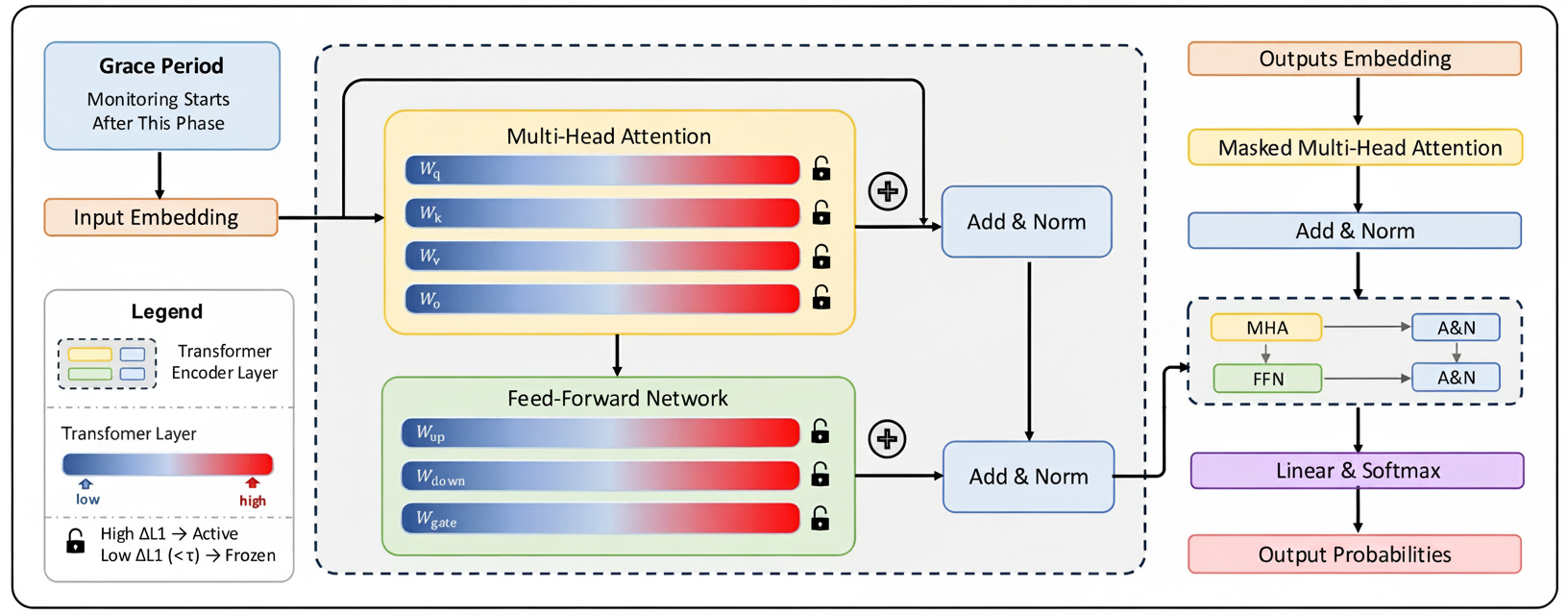} 
\caption{\textit{GradES} architecture. Color indicate gradient magnitude changes for each weight matrix. Components with low gradient changes (below threshold $\tau$) are stopped (lock icons) while others continue training.}
\label{fig:GradES_architecture} 
\end{figure}

Our experiments across five LLMs, varied from 0.6B to 14B parameters, demonstrate that \textit{GradES} reduces fine-tuning time by 50\% while maintaining or improving accuracy across eight benchmarks. Attention projections consistently stabilize 2--3 times faster than MLP components, with key and value projections stopping earliest which is a pattern that validates our component specific (Attention or MLP) approach over early stopping. \textit{GradES} can be seamlessly integrated with optimizers such as Adam \cite{kingma2017adam}, SGD \cite{robbins1951stochastic}, and parameter-efficient fine-tuning methods such as LoRA \cite{hu2022lora}, while complementing weight decay by preventing overfitting in converged components as weight decay continues, regularizing active components.

\textbf{Contributions.} We make the following contributions:

\begin{itemize}
\item We propose \textit{GradES}, a gradient-based early stopping method designed specifically for transformer and vision transformer architectures, eliminating expensive validation inference used for classic early stopping. 

\item We identified that gradient-based early stopping serves as an effective regularization technique, preventing overfitting in fast-converging Transformers' components while maintaining learning capacity in not yet converged completely, achieving improved accuracy up to 1.2\% on language tasks and 3.88\% on multimodal benchmarks and training time speed up of 1.57--7.22$\times$.

\item We validate the compatibility of \textit{GradES} with diverse optimization algorithms (e.g., Adam, SGD) and parameter-efficient fine-tuning methods (e.g., LoRA), showing consistent acceleration across training paradigms while preserving their respective computational and memory efficiency advantages.

\item We release our implementation as an open-source repository and PyPi package. It requires minimal modifications to existing training pipelines.
\end{itemize}

\section{Related Work}
\label{section:related_work}
\subsection{Transformer and VLM}
The Transformer architecture ~\cite{vaswani2017attention} consists of distinct weight matrices for queries, keys, values, output projections, and MLP components. 
Vision-Language Models (VLMs) have evolved from contrastive learning frameworks such as CLIP~\cite{radford2021learning}, which learns joint embeddings for zero-shot visual recognition, to architectures that enable complex multimodal reasoning. LLaVA~\cite{liu2023visual} integrated frozen vision encoders with large language models (LLMs) through learnable projection layers, refined by state-of-the-art models like Qwen2.5-VL-7B~\cite{bai2025qwen25vltechnicalreport}. Concurrently, efforts toward computational efficiency have produced frameworks like nanoVLM~\cite{wiedmann2025nanovlm}.
Interestingly, language encoders converge faster compared to vision layers which require extended optimization. Even within the language layers, attention and FFN layers converge at disparate rates. This variability motivates our \textit{GradES} algorithm.

\subsection{Parameter-Efficient Fine-tuning and Adaptive Training}
LoRA~\cite{hu2022lora} employs low-rank decomposition $\mathbf{B}\mathbf{A}$ where $r \ll \min(d_{\text{in}}, d_{\text{out}})$ (10,000$\times$ parameter reduction). 
A group of promising theoretical works established connections between gradient convergence patterns and optimization dynamics in neural networks. Arora et al. \cite{arora2019convergence} analyzes gradient descent convergence in deep linear networks, while Chatterjee et al. \cite{chatterjee2022convergence} extends these results to general deep architectures, demonstrating that different components exhibit distinct convergence trajectories. Nguegnang et al. \cite{nguegnang2021convergence} characterizes convergence rates for gradient descent in linear networks, providing theoretical foundations for component convergence monitoring. The role of gradient-based stopping criteria has been formalized by Patel et al. \cite{patel2020stopping}, who establish strong convergence guarantees when gradients fall below thresholds. Additionally, Gess et al. \cite{gess2024exponential} demonstrate exponential convergence rates for momentum methods in over parameterized settings, suggesting that convergence patterns are robust across different optimizers.

Building on these theoretical insights, several methods exploit convergence patterns for training efficiency. AutoFreeze~\cite{liu2021autofreeze} accelerates BERT fine-tuning by 2.55$\times$ through automatic layer freezing based on gradient patterns, though it operates only at layer granularity. AFLoRA~\cite{liu2024aflora} introduces gradual freezing of LoRA adapters using convergence scores, achieving 9.5$\times$ parameter reduction and 1.86$\times$ speedup. LoRAF~\cite{zhang2025lora} prevents catastrophic forgetting by selectively freezing LoRA matrices based on their convergence state, using 95\% fewer parameters than standard LoRA. However, these approaches either require coarse layer level decisions (AutoFreeze), additional architectural components like routers or masks (AFLoRA, LoRAF), or coupling with specific fine-tuning methods. 

\textit{GradES} advances this line of work by directly leveraging the heterogeneous convergence phenomenon~\cite{foster2018uniform} at the granularity of individual weight matrices. \textit{GradES} uses gradient magnitudes to detect when components converge~\cite{gao2024gradient}. It then freezes individual components as they reach convergence, adapting to each component's learning rate without requiring schedules.

\subsection{Early Stopping}
Early stopping represents an established regularization technique with extensive theoretical and practical foundations. Prechelt et al. ~\cite{prechelt1998early} provided practical guidelines for validation-based stopping criteria, while theoretical analyses by Yao et al. \cite{yao2007early} demonstrated that early stopping acts as implicit regularization in gradient descent learning. Recent theoretical advances by Ji et al. \cite{ji2021earlystopped} have proven the consistency of early stopping in neural networks, establishing rigorous convergence guarantees. Beyond traditional validation-based approaches, gradient-based early stopping by Mahsereci et al.~\cite{mahsereci2017early} and Pflug et al.~\cite{pflug1990online} monitors gradient magnitudes to detect convergence without requiring validation data. Recent work has explored more sophisticated stopping strategies: instance dependent early stopping by Yuan et al.~\cite{instance2025early} adapts termination per training example to reduce overfitting on easy samples, while correlation based methods by Ferro et al.~\cite{correlating2024early} combine multiple online indicators for more robust stopping decisions. Methods for noisy settings have also emerged, including strategies for learning with label noise by Bai et al.~\cite{bai2021understanding} and stopping without validation data by Yuan et al.~\cite{yuan2025earlystoppinglabelnoise}.

All these approaches apply stopping decisions globally, either terminating all training simultaneously or operating at the instance level. \textit{GradES} differs fundamentally by applying gradient-based convergence detection at the individual weight matrix level within transformers, allowing different components to stop at their natural convergence points while others continue learning. This approach will effectively bridge early stopping with fine-grained regularization.

\section{Method}
\label{section:method}
\subsection{GradES Algorithm}
We analyzed the magnitude of gradient changes of weight matrices across Transformer layers during fine-tuning and discovered different convergence rates for different components within Transformers' Layers. For each weight matrix $\mathbf{W}^{(l)}$ in layer $l \in \{1, \ldots, L\}$, we track the element-wise $L_1$ norm of gradients at training step $t$, as it is presented in Equation ~\ref{equation:gradient_norm_definition}.
\begin{equation}
\label{equation:gradient_norm_definition}
G_{\mathbf{W}}^{(l)}(t) = \|\nabla \mathbf{W}^{(l)}_t - \nabla \mathbf{W}^{(l)}_{t-1}\|_1 = \sum_{i=1}^{m} \sum_{j=1}^{n} |(\nabla \mathbf{W}^{(l)}_t)_{ij} - (\nabla \mathbf{W}^{(l)}_{t-1})_{ij}|
\end{equation}
where $\mathbf{W}^{(l)} \in \{\mathbf{W}_q^{(l)}, \mathbf{W}_k^{(l)}, \mathbf{W}_v^{(l)}, \mathbf{W}_o^{(l)}, \mathbf{W}_{\text{gate}}^{(l)}, \mathbf{W}_{\text{up}}^{(l)}, \mathbf{W}_{\text{down}}^{(l)}\}$ represents the attention projection matrices ($\mathbf{W}_q$, $\mathbf{W}_k$, $\mathbf{W}_v$, $\mathbf{W}_o$) and MLP weight matrices ($\mathbf{W}_{\text{gate}}$, $\mathbf{W}_{\text{up}}$, $\mathbf{W}_{\text{down}}$) within each Transformer layer. The metric $G_{\mathbf{W}}^{(l)}(t)$ quantifies the total gradient flow through each component, revealing distinct convergence trajectories that motivate our gradient-based component level early stopping strategy.
\begin{algorithm}[htbp]
\caption{\textit{GradES}: Early Stopping based on Absolute Change of Gradient Matrix }
\label{alg:GradES}
\begin{algorithmic}[1]
\Require Pre-trained model $\mathcal{M}$ with $L$ layers, dataset $\mathcal{D}$, total steps $T$, grace period ratio $\alpha$, threshold $\tau$, learning rate $\eta$
\Ensure Fine-tuned model $\mathcal{M}'$
\State \textbf{Initialize:} All parameters trainable, frozen set $\mathcal{F} \leftarrow \emptyset$
\State \textbf{Initialize:} Previous gradients $\nabla \mathbf{W}_{t-1} \leftarrow 0$ for all $\mathbf{W}$
\State $t_{\text{grace period}} \leftarrow \lceil \alpha \cdot T \rceil$ \Comment{Start monitoring after grace period}
\State The grace period is computed as a fraction $\alpha$ of the total training steps $T$.
\For{training step $t = 1$ to $T$}
    \State Sample $\mathcal{B} \sim \mathcal{D}$; compute loss $\mathcal{L}(\mathcal{B})$ and gradients
    
    \If{$t > t_{\text{grace period}}$} \Comment{Monitor after grace period}
        \For{each $\mathbf{W} \in$ all layers : $\mathbf{W} \notin \mathcal{F}$}
            \State $G_{\mathbf{W}}(t) = \sum_{i,j} |(\nabla \mathbf{W}_t)_{ij} - (\nabla \mathbf{W}_{t-1})_{ij}|$
            \State \textbf{if} $G_{\mathbf{W}}(t) < \tau$ \textbf{then} $\mathcal{F} \leftarrow \mathcal{F} \cup \{\mathbf{W}\}$
        \EndFor
    \EndIf
    \For{each projection matrix $\mathbf{W}$}
        \If{$\mathbf{W} \notin \mathcal{F}$} \Comment{Update only active parameters}
            \State $\mathbf{W} \leftarrow \mathbf{W} - \eta \cdot \nabla \mathbf{W}$
        \Else
            \State Skip update (but gradient still flows through)
        \EndIf
    \EndFor
    
    \For{each $\mathbf{W} \in$ all layers} \Comment{Store gradients for next step}
        \State $\nabla \mathbf{W}_{t-1} \leftarrow \nabla \mathbf{W}_t$
    \EndFor

    \If{all parameters frozen} \Comment{i.e., all $\mathbf{W} \in \mathcal{F}$}
        \State \textbf{break} \Comment{Early stopping}
    \EndIf
\EndFor
\State $\mathcal{M}' \leftarrow$ Update model with modified projection matrices $\mathbf{W}$
\State \Return Fine-tuned model $\mathcal{M}'$
\end{algorithmic}
\end{algorithm}

Algorithm~\ref{alg:GradES} formalizes, \textit{GardES}, our gradient-based early stopping procedure. \textit{GardES} algorithm introduces three key innovations that distinguish it from traditional early stopping approaches:

\textbf{Component-level convergence detection.} Unlike conventional methods that monitor global validation accuracy, \textit{GradES} tracks individual weight matrices $\mathbf{W}^{(l)} \in \{\mathbf{W}_q^{(l)}, \mathbf{W}_k^{(l)}, \mathbf{W}_v^{(l)}, \mathbf{W}_o^{(l)}, \mathbf{W}_{\text{gate}}^{(l)}, \mathbf{W}_{\text{up}}^{(l)}, \mathbf{W}_{\text{down}}^{(l)}\}$ within each layer $l$. We employ the $L_1$ gradient norm $G_{\mathbf{W}}(t) = \|\nabla \mathbf{W}\|_1$ as our convergence metric, chosen for its computational efficiency compared to $L_2$ norms. When $G_{\mathbf{W}}(t) < \tau$, we consider the component converged stop updating the weight matrix (lines 8-14).

\textbf{Adaptive grace period strategy.} The initial $t_{\text{grace period}} = \lceil \alpha T \rceil$ steps (with $\alpha = 0.5$ in our experiments) allow all components to escape their pre-trained initialization before convergence monitoring begins. This prevents terminating training of components prematurely that may appear initially converged but require substantial adaptation for the downstream task. The grace period duration depends on the total training examples and scales proportionally with the total training budget $T$.

\textbf{Gradient flow preservation.} A critical design choice is maintaining gradient computation through Converged matrices (line 12). While converged components do not receive parameter updates (lines 17-22), they continue to propagate gradients to earlier layers. This ensures that active components receive proper gradient signals throughout training, preventing the gradient flow disruption that would occur with complete component removal, like pruning. 

The algorithm terminates when all components satisfy the convergence criterion (Line 24), eliminating unnecessary computation on converged parameters. We provide formal convergence guarantees and theoretical analysis in Appendix~\ref{section:convergence}.

\subsection{\textit{GradES} for Low-Rank Adaptation}
LoRA is one of the most common approaches used for fine-tuning.  Since LoRA constrains parameters to a low-dimensional subspace, gradient dynamics in this reduced space exhibit fundamentally different convergence properties than full fine-tuning. When applying \textit{GradES} to LoRA ~\cite{hu2022lora} fine-tuning, we monitor gradient magnitudes in the low-rank space rather than the full parameter space. Let $l \in \{1, \ldots, L\}$ denote the layer index where $L$ is the total number of layers. Within each transformer layer $l$, we apply LoRA decomposition to individual weight matrices $\mathbf{W}^{(l)} \in \{\mathbf{W}_q^{(l)}, \mathbf{W}_k^{(l)}, \mathbf{W}_v^{(l)}, \mathbf{W}_o^{(l)}, \mathbf{W}_{\text{gate}}^{(l)}, \mathbf{W}_{\text{up}}^{(l)}, \mathbf{W}_{\text{down}}^{(l)}\}$, where the first four correspond to attention projections and the latter three to MLP components.

For each weight matrix $\mathbf{W}^{(l)} \in \mathbb{R}^{d_{\text{out}} \times d_{\text{in}}}$, where $d_{\text{out}}, d_{\text{in}}$ are the output and input dimensions, in layer $l$, the LoRA weight is:
\begin{equation}
\mathbf{W}_{\text{adapted}}^{(l)} = \mathbf{W}_{\text{frozen}}^{(l)} + \mathbf{B}_{\mathbf{W}}^{(l)}\mathbf{A}_{\mathbf{W}}^{(l)}
\end{equation}
where $\mathbf{B}_{\mathbf{W}}^{(l)} \in \mathbb{R}^{d_{\text{out}} \times r}$ and $\mathbf{A}_{\mathbf{W}}^{(l)} \in \mathbb{R}^{r \times d_{\text{in}}}$ are the trainable low-rank matrices with rank $r \ll \min(d_{\text{out}}, d_{\text{in}})$. 

For each individual LoRA matrix $\mathbf{W}$ in layer $l$, we track convergence by monitoring the combined gradient magnitude:
\begin{equation}
G_{\mathbf{W}}^{(l)}(t) = \|\nabla\mathbf{A}_{\mathbf{W}}^{(l)}\|_1 + \|\nabla\mathbf{B}_{\mathbf{W}}^{(l)}\|_1
\end{equation}
where $\nabla\mathbf{A}_{\mathbf{W}}^{(l)}$ and $\nabla\mathbf{B}_{\mathbf{W}}^{(l)}$ denote the gradients of the low-rank matrices at training step $t$, and $\|\cdot\|_1$ denotes the element-wise $L_1$ norm.

The freezing operates at the matrix level, enabling precise control. After the grace period period $t > t_{\text{grace period}}$, we independently freeze each LoRA matrix when:
\begin{equation}
G_{\mathbf{W}}^{(l)}(t) < \tau_r \quad \text{for } \mathbf{W} \in \{\mathbf{W}_q, \mathbf{W}_k, \mathbf{W}_v, \mathbf{W}_o, \mathbf{W}_{\text{gate}}, \mathbf{W}_{\text{up}}, \mathbf{W}_{\text{down}}\}
\end{equation}
where $\tau_r$ is the convergence threshold adjusted for the reduced parameter count. Once a specific matrix reaches convergence, we stop updating its corresponding $\mathbf{A}_{\mathbf{W}}^{(l)}$ and $\mathbf{B}_{\mathbf{W}}^{(l)}$ while continuing to compute gradients through them for backpropagation. Training terminates when all LoRA matrices across all layers are frozen.

\section{Experimental Setup}
\label{section:experiments}

\subsection{Models}
We evaluate \textit{GradES} on the 5 most popular language models and 2 vision language model on huggingface~\cite{huggingface_hub} (at the time of conducting this experiment) to represent different parameter scales, quantization strategies, and architectural designs. Our language model selection spans three orders of magnitude in parameter count (0.6B to 14B) to investigate the scalability of our approach.  We conducted all experiments using 4-bit quantized models, due to hardware limitations of the experimental platform. Specifically, we employ: \textbf{Qwen3-14B}~\cite{qwen2024qwen25}; \textbf{Microsoft Phi4-14B}~\cite{phi42024technicalreport}; \textbf{Llama-3.1-8B-Instruct}~\cite{dubey2024llama3}; \textbf{Mistral-7B-Instruct-v0.3}~\cite{jiang2023mistral}; and the compact \textbf{Qwen3-0.6B}~\cite{qwen2024qwen25}. To assess performance on multimodal tasks, we also evaluate two vision-language models: \textbf{Qwen2.5-VL-7B}~\cite{bai2025qwen25vltechnicalreport} and the smaller \textbf{nanoVLM}~\cite{wiedmann2025nanovlm}.

\subsection{Benchmark and Evaluation Metrics}
We evaluate \textit{GradES} on eight widely used benchmarks covering different aspects of language understanding. For reasoning tasks, we use BoolQ~\cite{clark2019boolq} for boolean question answering requiring complex reasoning, PIQA~\cite{bisk2020piqa} for physical commonsense reasoning about everyday situations, SIQA~\cite{sap2019socialiqa} for social commonsense reasoning about human interactions, and HellaSwag~\cite{zellers2019hellaswag} for commonsense inference about plausible continuations. For knowledge-intensive task completion, we employ OpenBookQA~\cite{mihaylov2018can} for science questions requiring multi-hop reasoning, ARC-Easy and ARC-Challenge~\cite{clark2018think} for grade school science questions at two difficulty levels, and WinoGrande~\cite{sakaguchi2021winogrande} for pronoun resolution requiring commonsense knowledge.
We measure both task accuracy and computational efficiency to provide a rounded view of \textit{GradES}'s benefits. Average task accuracy is measured by accuracy on each benchmark's test set. For computational efficiency metrics, we track training time on identical hardware and floating point operations (FLOPs) computed using PyTorch profiler~\cite{paszke2019pytorch}.

We evaluate vision transformers using the LMMs-Eval harness~\cite{zhang2024lmmsevalrealitycheckevaluation, lmms_eval2024} to assess nanoVLM training accuracy. We conduct fine-tuning experiments on Qwen2.5-VL-7B and evaluate performance across three established benchmarks: GQA~\cite{hudson2019gqanew} for compositional visual reasoning, VQAv2~\cite{goyal2017making} for robust visual question answering, and COCO Captions~\cite{chen2015cococaptionsdata} for image captioning capabilities.

\subsection{PEFT and Early Stopping Methods}
We evaluate \textit{GradES} against established fine-tuning paradigms to demonstrate its benefits. Our baseline methods comprise Full Parameter Fine-tuning (FP), which updates all model parameters without constraints; and LoRA~\cite{hu2022lora} to assess the composability of our method. In particular, we apply \textit{GradES} to both full fine-tuning (FP+GradES) and LoRA (LoRA+GradES), yielding six distinct configurations for comprehensive evaluation. For validation-based early stopping baselines (FP+ES and LoRA+ES), we perform validation checks at 5\% intervals throughout training. Terminating training when validation loss fails to improve for consecutive checkpoints. We presents the overhead of early stopping compared to \textit{GradES} in ~\ref{tab:timing-flops-comparison}. More detailed experiment settings and hyperparameter selection are provided in Appendix~\ref{section:code_availability}.

\section{Results}
\label{section:results}
\subsection{Accuracy on Benchmarks}
We evaluate \textit{GradES} against standard full-parameter (FP) fine-tuning and LoRA across five language models ranging from 0.6B to 14B parameters on eight commonsense reasoning benchmarks. As shown in Table~\ref{tab:LLM_accuracy_results}, \textit{GradES} consistently improves upon the baseline early stopping (ES) method across both fine-tuning methods (FP and LoRA). For full-parameter fine-tuning on larger models (14B), full-parameter with \textit{GradES} achieves the highest average accuracy on Qwen3 (90.81\%) and Phi4 (91.94\%), demonstrating consistent improvements over full-parameter fine-tuning (90.80\% and 91.93\% respectively). Its impact is more significant on smaller models, where LoRA (ES) and LoRA (GradES) on Qwen3 0.6B achieve 67.37\% and 67.30\% average accuracy, substantially outperforming standard full-parameter methods ($\sim$66.5\%). Notably, the choice of base fine-tuning method (Full-parameter vs. LoRA) exhibits its model-independent behavior, while full-parameter fine-tuning methods perform competitively on larger models (Qwen3 14B, Phi4 14B), LoRA variants showed higher accuracy on mid-sized models, with LoRA (ES) achieving 86.27\% on Mistral-7B compared to 75.80\% for standard full-parameter fine-tuning, a remarkable 10.47 percentage point improvement. \textit{GradES} shows particular strength on specific benchmarks, achieving best results on Winograde compared to other methods (Qwen3 14B: 84.77\%), PIQA (Phi4 14B: 92.60\%). 
\begin{table}[H]
\centering
\small
\caption{Comparison of the accuracy for different fine-tuning methods on five different language models. Values are reported in percentages, and the best one in each category is highlighted in bold.}
\label{tab:LLM_accuracy_results}
\resizebox{\textwidth}{!}{
\begin{tabular}{@{}llccccccccc@{}}
\toprule
\textbf{Model} & \textbf{Method} & \textbf{BoolQ} & \textbf{PIQA} & \textbf{SIQA} & \textbf{HellaSwag} & \textbf{Winograde} & \textbf{OpenBookQA} & \textbf{ARC-C} & \textbf{ARC-E} & \textbf{Avg.} \\
\midrule
\multirow{6}{*}{Qwen3 14B} 
    & Full Parameter       & 91.07 & 91.29 & 81.93 & 95.11 & 83.03 & 91.60 & 94.31 & 98.07 & 90.80 \\
    & FP+ES  & 91.07 & 91.08 & 81.99 & 95.05 & 83.03 & 91.40 & 93.98 & 98.07 & 90.71 \\
    & FP+GradES& \textbf{91.22} & 91.19 & 82.04 & 94.97 & 83.03 & \textbf{91.80} & \textbf{94.31} & 97.89 & \textbf{90.81} \\
    & LoRA                 & 90.86 & \textbf{91.24} & 81.68 & 95.21 & 83.35 & 91.40 & 94.31 & 97.19 & 90.65 \\
    & LoRA+ES           & 90.64 & 91.08 & 81.99 & \textbf{95.40} & 81.53 & 91.40 & 93.65 & 97.54 & 90.40 \\
    & LoRA+GradES         & 90.67 & 91.13 & \textbf{82.29} & 95.22 & \textbf{84.77} & 91.20 & 93.31 & 97.02 & 90.70 \\
\midrule
\multirow{6}{*}{Phi4 14B} 
    & Full Parameter    & 90.49 & 91.95 & 83.27 & 95.49 & 88.71 & 93.00 & 94.31 & 98.25 & 91.93 \\
    & FP+ES  & 90.34 & 92.11 & 82.55 & 95.50 & 88.56 & 93.00 & 94.31 & 98.07 & 91.80\\
    & FP+GradES& 90.31 & \textbf{92.60} & \textbf{83.06} & 95.43 & \textbf{88.95} & 92.60 & 94.31 & 98.25 & \textbf{91.94} \\
    & LoRA                 & 90.31 & 92.00 & 82.45 & 95.36 & 87.53 & 91.80 & \textbf{94.65} & 98.07 & 91.52 \\
    & LoRA+ES           & 90.61 & 92.22 & 82.60 & \textbf{95.44} & 87.37 & 92.00 & 94.31 & \textbf{98.25} & 91.60 \\
    & LoRA+GradES         & \textbf{90.49} & 92.60 & 82.40 & 95.38 & 87.37 & 91.60 & 94.65 & 98.07 & 91.57 \\
\midrule
\multirow{6}{*}{Qwen3 0.6B} 
    & Full Parameter     & 77.28 & 69.31 & 66.99 & 65.09 & 50.28 & 61.20 & 61.54 & 80.53 & 66.53\\
    & FP+ES & 77.06 & 68.99 & 67.09 & 65.16 & 49.57 & 61.00 & 63.21 & 81.23 & 66.66\\
    & FP+GradES& 77.03 & \textbf{71.38} & 66.84 & 64.26 & \textbf{51.62} & 62.40 & 61.20 & 79.65 & 66.80 \\
    & LoRA                 & \textbf{79.14} & 69.15 & \textbf{69.14} & 67.94 & 49.09 & \textbf{66.20} & 60.54 & 77.19 & 67.30 \\
    & LoRA+ES           & 79.11 & 69.10 & 68.83 & \textbf{68.18} & 49.64 & 65.00 & 61.20 & 77.89 & \textbf{67.37} \\
    & LoRA+GradES         & 78.56 & 69.42 & 68.17 & 68.20 & 49.41 & 65.40 & \textbf{61.54} & 77.72 & 67.30 \\
\midrule
\multirow{6}{*}{Llama-3.1-8B} 
    & Full Parameter    & 89.27 & 88.08 & 81.01 & 94.40 & 81.93 & 85.00 & 83.61 & 92.46 & 86.97\\
    & FP+ES  & 89.24 & 87.81 & 80.86 & 94.42 & 82.56 & 85.80 & 82.61 & 92.11 & 86.93\\
    & FP+GradES& 88.87 & \textbf{88.25} & 80.45 & 94.23 & 82.79 & 84.80 & \textbf{83.95} & \textbf{92.81} & \textbf{87.02} \\
    & LoRA                 & 87.98 & 87.65 & 79.73 & 94.31 & 80.35 & \textbf{87.20} & 83.28 & 91.40 & 86.49 \\
    & LoRA+ES           & 88.62 & 88.19 & 79.32 & 94.20 & 80.35 & 87.20 & 83.28 & 91.75 & 86.62\\
    & LoRA+GradES         & \textbf{88.78} & 88.03 & \textbf{79.68} & \textbf{94.43} & \textbf{81.69} & 87.00 & 83.95 & 90.53 & 86.76 \\
\midrule
\multirow{6}{*}{Mistral-7B} 
    & Full Parameter& 85.26 & 80.25 & 77.33 & 83.71 & 66.30 & 75.60 & 62.54 & 75.44 & 75.80 \\
    & FP (ES)  & 85.32 & 80.09 & 76.20 & 83.95 & 57.54 & 74.60 & 64.21 & 74.39 & 74.54\\
    & FP+GradES & 85.38 & 78.56 & 75.79 & 84.14 & 66.47 & 76.40 & 61.20 & 75.10 & 75.38 \\
    & LoRA                 & \textbf{89.33} & 87.43 & 79.79 & \textbf{94.95} & 79.72 & 84.20 & 76.25 & 88.42 & 85.01 \\
    & LoRA+ES           & 88.93 & \textbf{88.30} & \textbf{81.01} & 94.69 & \textbf{82.24} & \textbf{85.00} & 79.60 & 90.35 & \textbf{86.27} \\
    & LoRA+GradES         & 89.36 & 88.03 & 80.71 & 94.69 & 80.90 & 84.40 & \textbf{81.61} & \textbf{89.65} & 86.17 \\
\bottomrule
\end{tabular}
}
\end{table}

Table~\ref{tab:VLM_finetuning_accuracy_results} presents \textit{GradES} performance on vision-language tasks using Qwen2.5-VL-7B. \textit{GradES} consistently improves both full-parameter and LoRA fine-tuning, with LoRA+GradES achieving the highest average accuracy (70.6\%). LoRA methods substantially outperform full-parameter approaches on image captioning (54.44\% vs. 41.61\% for COCO Captions), while LoRA+GradES attains best performance on VQAv2 (81.24\%). These results demonstrate that \textit{GradES} effectively enhances vision-language model performance, particularly when combined with PEFT methods.

\begin{table}[H]
\centering
\small
\caption{Performance comparison of \textit{GradES} against standard fine-tuning methods on vision-language benchmarks. Results show accuracy (\%) for Qwen2.5-VL-7B across visual reasoning (GQA), question answering (VQAv2), and image captioning (COCO Cap) tasks.}
\label{tab:VLM_finetuning_accuracy_results}
\begin{tabular}{@{}llcccc@{}}
\toprule
\textbf{Model} & \textbf{Method} & \textbf{GQA} & \textbf{VQAv2} & \textbf{COCO Cap} & \textbf{Avg.} \\
\midrule
\multirow{4}{*}{Qwen2.5-VL-7B}
& Full Parameter & 75.69 & 81.0 & 41.38 & 66.08 \\
& FP+GradES & 76.08 & 80.81 & 41.61 & 66.20 \\
& LoRA & 76.49 & 81.01 & 53.22 & 70.24 \\
& LoRA+GradES & 76.13 & 81.24 & 54.44 & 70.6 \\
\bottomrule
\end{tabular}%
\end{table}

\begin{table}[H]
\centering
\caption{Performance of \textit{GradES} on nanoVLM training across multimodal reasoning benchmarks. Results show accuracy (\%) for full-parameter fine-tuning with and without \textit{GradES} across perception, reasoning, and knowledge-based tasks.}
\label{tab:VLM_training_accuracy_results}
\begin{tabular}{@{}lcc@{}}
\toprule
\textbf{Benchmark} & \textbf{Training} & \textbf{Training+GradES} \\
\midrule
Coarse Perception & 38.87 & 42.42 \\
Fine-grained Perception & 22.40 & 29.82 \\
Instance Reasoning & 36.07 & 36.42 \\
Logical Reasoning & 28.86 & 37.21 \\
Math & 27.60 & 28.60 \\
Science \& Technology & 31.10 & 33.74 \\
\midrule
\textbf{Avg.} & \textbf{30.82} & \textbf{34.70} \\
\bottomrule
\end{tabular}%
\end{table}

Table~\ref{tab:VLM_training_accuracy_results} demonstrates \textit{GradES}'s effectiveness on nanoVLM training across diverse multimodal benchmarks. Training with \textit{GradES} achieves substantial improvements over standard training, with an average accuracy increase of 3.88 percentage points (34.70\% vs. 30.82\%). Most notably, \textit{GradES} yields significant gains on Fine-grained Perception (+7.42\% and Logical Reasoning (+8.35\%), while maintaining consistent improvements across all evaluated domains. These results confirm \textit{GradES}'s ability to enhance vision-language model training, particularly for tasks requiring detailed visual understanding and complex reasoning.

\subsection{Training Efficiency}
Tables~\ref{tab:timing-flops-comparison} and~\ref{tab:VLM-timing-flops-comparison} present computational efficiency metrics across language and vision-language models. \textit{GradES} demonstrates consistent efficiency improvements over both baseline and traditional early stopping (ES) approaches. For language models, FP+GradES achieves speedups of 1.32--1.64$\times$ while reducing FLOPs by 29--45\%, with the most significant gains on larger models (Qwen3-14B: 1.51$\times$ speedup, 0.55$\times$ FLOPs). In contrast, traditional ES paradoxically slows training (0.59--0.76$\times$) despite FLOPs reduction, highlighting the overhead of frequent validation and convergence monitoring. Vision-language models show similar trends, with FP+GradES achieving 1.17$\times$ speedup and 12\% FLOPs reduction on Qwen2.5-VL-7B.

The efficiency benefits extend to parameter-efficient methods, where LoRA+GradES consistently outperforms standard LoRA across all models. For language models, LoRA+GradES achieves speedups of 2.66--2.87$\times$ over full-parameter baselines while reducing LoRA's computational overhead from 2.34--2.43$\times$ to 1.88--2.29$\times$ FLOPs. On vision-language tasks, LoRA+GradES delivers the highest speedup (1.80$\times$), reducing training time by 44\%. Remarkably, these computational savings preserve model quality—\textit{GradES} variants achieve the highest accuracies while requiring substantially fewer FLOPs than standard fine-tuning. These results establish gradient-guided early stopping as a practical solution for resource-constrained deployment, demonstrating consistent efficiency gains across model architectures (Qwen, Phi, Llama, Mistral), scales (0.6B--14B parameters), and modalities (language and vision-language).

\begin{table}[H]
\centering
\caption{Training time and computational cost comparison of different fine-tuning methods on 5 different language models. Training time in seconds, FLOPs in floating point operations. Speedup and FLOPs ratios are computed relative to Full Parameter(Base) for all methods. The best one in each category is highlighted in bold.}
\label{tab:timing-flops-comparison}
\begin{tabular}{@{}llcccc@{}}
\toprule
\textbf{Model} & \textbf{Method} & \textbf{Training Time (s)} & \textbf{Speedup} & \textbf{FLOPs} & \textbf{FLOPs Ratio} \\
\midrule
\multirow{6}{*}{Qwen3-14B}
    & Full Parameter(Base) & 16,202 & 1.00$\times$ & $1.17 \times 10^{18}$ & 1.00$\times$ \\
    & FP+ES                & 22,466 & 0.72$\times$ & $8.76 \times 10^{17}$ & 0.75$\times$ \\
    & FP+GradES            & 10,721 & 1.51$\times$ & $6.43 \times 10^{17}$ & \textbf{0.55$\times$} \\
    & LoRA                 & 6,387  & 2.54$\times$ & $2.74 \times 10^{18}$ & 2.34$\times$ \\
    & LoRA+ES              & 23,932 & 0.68$\times$ & $2.74 \times 10^{18}$ & 2.34$\times$ \\
    & LoRA+GradES          & 5,643  & \textbf{2.87$\times$} & $2.43 \times 10^{18}$ & 2.08$\times$ \\
\midrule
\multirow{6}{*}{Phi4-14B}
    & Full Parameter(Base) & 14,627 & 1.00$\times$ & $1.12 \times 10^{18}$ & 1.00$\times$ \\
    & FP+ES                & 23,040 & 0.63$\times$ & $9.49 \times 10^{17}$ & 0.85$\times$ \\
    & FP+GradES            & 9,218  & 1.59$\times$ & $6.15 \times 10^{17}$ & \textbf{0.55$\times$} \\
    & LoRA                 & 6,030  & 2.43$\times$ & $2.69 \times 10^{18}$ & 2.40$\times$ \\
    & LoRA+ES              & 24,394 & 0.60$\times$ & $2.69 \times 10^{18}$ & 2.40$\times$ \\
    & LoRA+GradES          & 5,506  & \textbf{2.66$\times$} & $2.38 \times 10^{18}$ & 2.13$\times$ \\
\midrule
\multirow{6}{*}{Qwen3-0.6B}
    & Full Parameter(Base) & 6,550  & 1.00$\times$ & $3.68 \times 10^{16}$ & 1.00$\times$ \\
    & FP+ES                & 8,569  & 0.76$\times$ & $2.76 \times 10^{16}$ & 0.75$\times$ \\
    & FP+GradES            & 4,018  & 1.63$\times$ & $2.03 \times 10^{16}$ & \textbf{0.55$\times$} \\
    & LoRA                 & 892    & \textbf{7.34$\times$} & $8.94 \times 10^{16}$ & 2.43$\times$ \\
    & LoRA+ES              & 6,155  & 1.06$\times$ & $8.94 \times 10^{16}$ & 2.43$\times$ \\
    & LoRA+GradES          & 907    & 7.22$\times$ & $8.43 \times 10^{16}$ & 2.29$\times$ \\
\midrule
\multirow{6}{*}{Llama-3.1-8B}
    & Full Parameter(Base) & 9,541  & 1.00$\times$ & $7.45 \times 10^{17}$ & 1.00$\times$ \\
    & FP+ES                & 16,129 & 0.59$\times$ & $6.70 \times 10^{17}$ & 0.90$\times$ \\
    & FP+GradES            & 5,832  & 1.64$\times$ & $4.10 \times 10^{17}$ & \textbf{0.55$\times$} \\
    & LoRA                 & 3,737  & 2.55$\times$ & $1.58 \times 10^{18}$ & 2.12$\times$ \\
    & LoRA+ES              & 15,499 & 0.62$\times$ & $1.58 \times 10^{18}$ & 2.12$\times$ \\
    & LoRA+GradES          & 3,370  & \textbf{2.83$\times$} & $1.40 \times 10^{18}$ & 1.88$\times$ \\
\midrule
\multirow{6}{*}{Mistral-7B}
    & Full Parameter(Base) & 9,256  & 1.00$\times$ & $6.38 \times 10^{17}$ & 1.00$\times$ \\
    & FP+ES                & 14,521 & 0.64$\times$ & $6.38 \times 10^{17}$ & 1.00$\times$ \\
    & FP+GradES            & 6,996  & 1.32$\times$ & $4.51 \times 10^{17}$ & \textbf{0.71$\times$} \\
    & LoRA                 & 3,752  & 2.47$\times$ & $1.51 \times 10^{18}$ & 2.37$\times$ \\
    & LoRA+ES              & 11,159 & 0.83$\times$ & $1.51 \times 10^{18}$ & 2.37$\times$ \\
    & LoRA+GradES          & 3,259  & \textbf{2.84$\times$} & $1.32 \times 10^{18}$ & 2.07$\times$ \\
\bottomrule
\end{tabular}
\end{table}

\begin{table}[H]
\centering
\caption{Time and FLOPs comparison of \textit{GradES} on vision-language model for Qwen2.5-VL-7B across different fine-tuning methods.}
\label{tab:VLM-timing-flops-comparison}
\begin{tabular}{@{}llcccc@{}}
\toprule
\textbf{Model} & \textbf{Method} & \textbf{Training Time (s)} & \textbf{Speedup} & \textbf{FLOPs} & \textbf{FLOPs Ratio} \\
\midrule
\multirow{4}{*}{Qwen2.5-VL-7B}
    & Full Parameter  & 157,239 & 1.00$\times$ & $2.27 \times 10^{19}$ & 1.00$\times$ \\
    & FP+GradES       & 134,686 & 1.17$\times$ & $1.99 \times 10^{19}$ & \textbf{0.88$\times$} \\
    & LoRA            & 103,466 & 1.52$\times$ & $2.80 \times 10^{19}$ & 1.23$\times$ \\
    & LoRA+GradES     & 87,572  & \textbf{1.80$\times$} & $2.52 \times 10^{19}$ & 1.11$\times$ \\
\bottomrule
\end{tabular}%
\end{table}

\subsection{Ablation Studies}
We first investigate the necessity of the grace period $\alpha$. When we remove the grace period entirely ($\alpha = 0$), training terminates immediately after initialization. As illustrated in Figure~\ref{fig:within_layer_gradient}, gradient norms across all components remain small during the warm-up phase, falling below the convergence threshold $\tau$. This observation confirms that the grace period is crucial for our method. 

Then we evaluate the impact of the grace period $\alpha \in \{0.1, 0.2, 0.3, 0.4, 0.5, 0.6\}$ and convergence threshold $\tau \in \{1.5, 3.0, 4.5, 6.0, 7.5, 9.0\}$ on both model performance and training efficiency. All experiments are conducted on the Qwen-14B model across eight benchmark tasks: BoolQ, PIQA, SIQA, HellaSwag, WinoGrande, OpenBookQA, ARC-Challenge, and ARC-Easy. Table~\ref{tab:ablation-performance} reports the average evaluation accuracy across these benchmarks, while Table~\ref{tab:ablation-time} presents the corresponding training times.

Our results reveal that optimal performance (92.81\% average accuracy) is achieved with $\tau = 1.5$ and $\alpha = 0.5$, suggesting that a moderate grace period combined with a convergence threshold yields the best generalization. Conversely, the most significant speedup occurs at $\tau = 9.0$ and $\alpha = 0.1$, where the large threshold enables early stopping of most components and the minimal grace period starts of freezing.

\begin{table}[H]
\centering
\caption{Accuracy values are shown as percentages for threshold ($\tau$) and grace period ($\alpha$) parameters, with the best result highlighted in bold.}
\label{tab:ablation-performance}
\begin{tabular*}{\textwidth}{@{\extracolsep{\fill}}lcccccc@{}}
\toprule
\textbf{$\tau$ / $\alpha$} & \textbf{0.1} & \textbf{0.2} & \textbf{0.3} & \textbf{0.4} & \textbf{0.5} & \textbf{0.6} \\
\midrule
1.5 & 90.64 & 90.59 & 90.77 & 90.68 & \textbf{90.93} & 90.83 \\
3.0 & 90.58 & 90.70 & 90.51 & 90.61 & 90.54 & 90.83 \\
4.5 & 90.50 & 90.69 & 90.62 & 90.61 & 90.74 & 90.80 \\
6.0 & 90.46 & 90.68 & 90.67 & 90.58 & 90.68 & 90.78 \\
7.5 & 90.24 & 90.27 & 90.58 & 90.61 & 90.70 & 90.78 \\
9.0 & 90.01 & 90.35 & 90.34 & 90.66 & 90.71 & 90.71 \\
\bottomrule
\end{tabular*}
\end{table}

\begin{table}[H]
\centering
\caption{Fine-tuning time for threshold ($\tau$) and grace period ($\alpha$) parameters, with the fastest time highlighted in bold.}
\label{tab:ablation-time}
\begin{tabular*}{\textwidth}{@{\extracolsep{\fill}}lcccccc@{}}
\toprule
\textbf{$\tau$ / $\alpha$} & \textbf{0.1} & \textbf{0.2} & \textbf{0.3} & \textbf{0.4} & \textbf{0.5} & \textbf{0.6} \\
\midrule
1.5 & 10726.74 & 10592.54 & 10629.99 & 10840.85 & 11187.48 & 11774.73 \\
4.5 & 8925.71 & 9269.45 & 9317.26 & 9789.08 & 10273.89 & 11318.94 \\
7.5 & 8038.52 & 8337.48 & 8599.51 & 9262.74 & 9962.62 & 11087.04 \\
9.0 & \textbf{7393.14} & 7746.00 & 8036.02 & 9130.12 & 9477.76 & 10916.65 \\
\bottomrule
\end{tabular*}
\end{table}

\section{Discussion}
\label{section:analysis}
\subsection{Convergence Patterns Across Different Models}
Figure~\ref{fig:cumulative_frozen_components} shows the progression of the matrix that converged across different model scales during training. After a warm-up period of $\alpha = 1000$ steps, our method begins freezing converged components based on gradient magnitude thresholds $\tau$ (model-specific values detailed in Appendix~\ref{section:hyperparameter}).

The difference in convergence rate reflects distinct architectural behaviors. Larger models (7B-14B) exhibit rapid convergence, with the majority of components frozen by step 1400—approximately 40\% through training. In contrast, the smaller Qwen-0.6B model demonstrates delayed convergence, with no components meeting the freezing criteria until step 1600. 

Notably, the threshold $\tau$ varies significantly between training methods, and full fine-tuning requires larger thresholds compared to LoRA adaptation. 

\begin{figure}[H] \centering \includegraphics[width=0.8\textwidth]{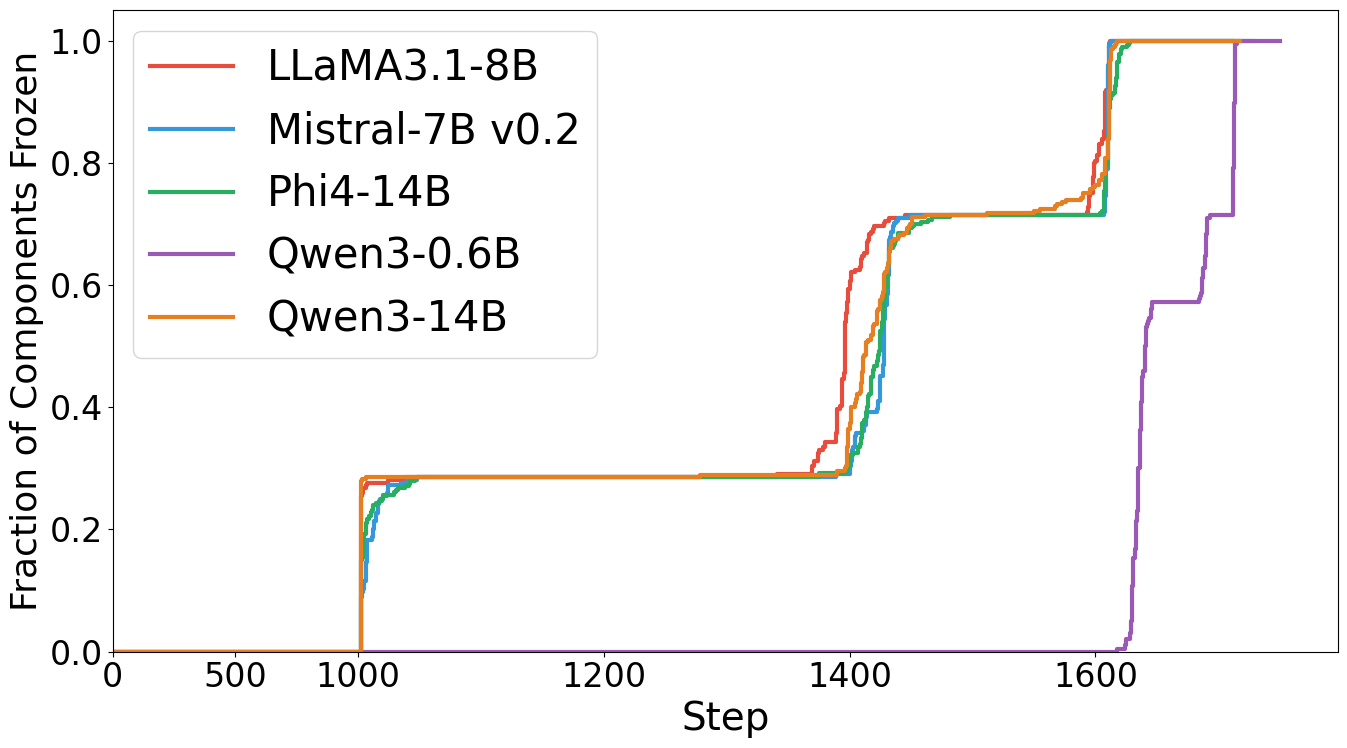} \caption{Cumulative frozen components during training across model scales. Fraction of weight matrices frozen over time for five different LLMs.} \label{fig:cumulative_frozen_components} 
\end{figure}

\subsection{Attention versus MLP components}
Figure~\ref{fig:all_component_gradient} presents the gradient norms $|\nabla W|$ across all weight matrices during fine-tuning of the Qwen-0.6B model. We compute the element-wise L1 norm for each weight matrix and report averaged values for two architectural components: MLP matrices ($\mathbf{W}_{\text{up}}$, $\mathbf{W}_{\text{down}}$; orange) and attention projection matrices ($\mathbf{W}_q$, $\mathbf{W}_k$, $\mathbf{W}_v$, $\mathbf{W}_o$; blue).

We have two key observations. First, the gradient trend reflects the cosine learning rate schedule: initial warmup drives increasing gradient magnitudes as the model adapts to the task distribution, while the subsequent cosine decay produces monotonically decreasing gradient norms, enabling smooth convergence from exploration to refinement.

Second, and more critically for our method, MLP weight matrices consistently maintain larger gradient norms compared to attention projection matrices throughout training. This persistent gap indicates that MLP parameters require more steps to converge, suggesting inefficiency under uniform training strategies. This observation directly motivates our approach: by dynamically allocating computational resources proportional to gradient magnitudes, we can accelerate convergence of slower learning components while maintaining overall model accuracy.

\begin{figure}[H]
\centering
\begin{subfigure}{0.48\textwidth}
    \centering
    \includegraphics[width=\textwidth]{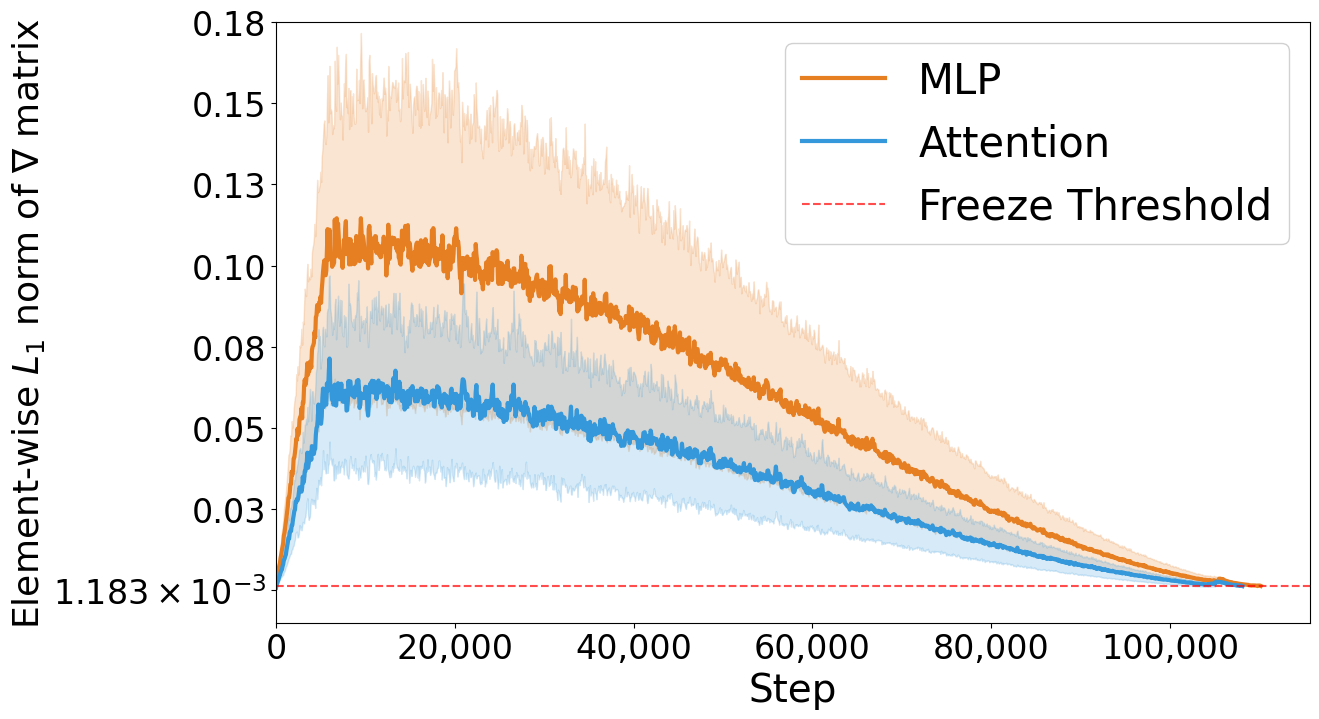}
    \caption{Gradient norm evolution during Qwen-0.6B fine-tuning. Element-wise L1 norms of weight gradients averaged across layers for MLP matrices (orange) and attention projections (blue). MLP matrices consistently exhibit larger gradient magnitudes throughout training, indicating slower convergence and motivating targeted computational allocation.}
    \label{fig:all_component_gradient}
\end{subfigure}
\hfill
\begin{subfigure}{0.48\textwidth}
    \centering
    \includegraphics[width=\textwidth]{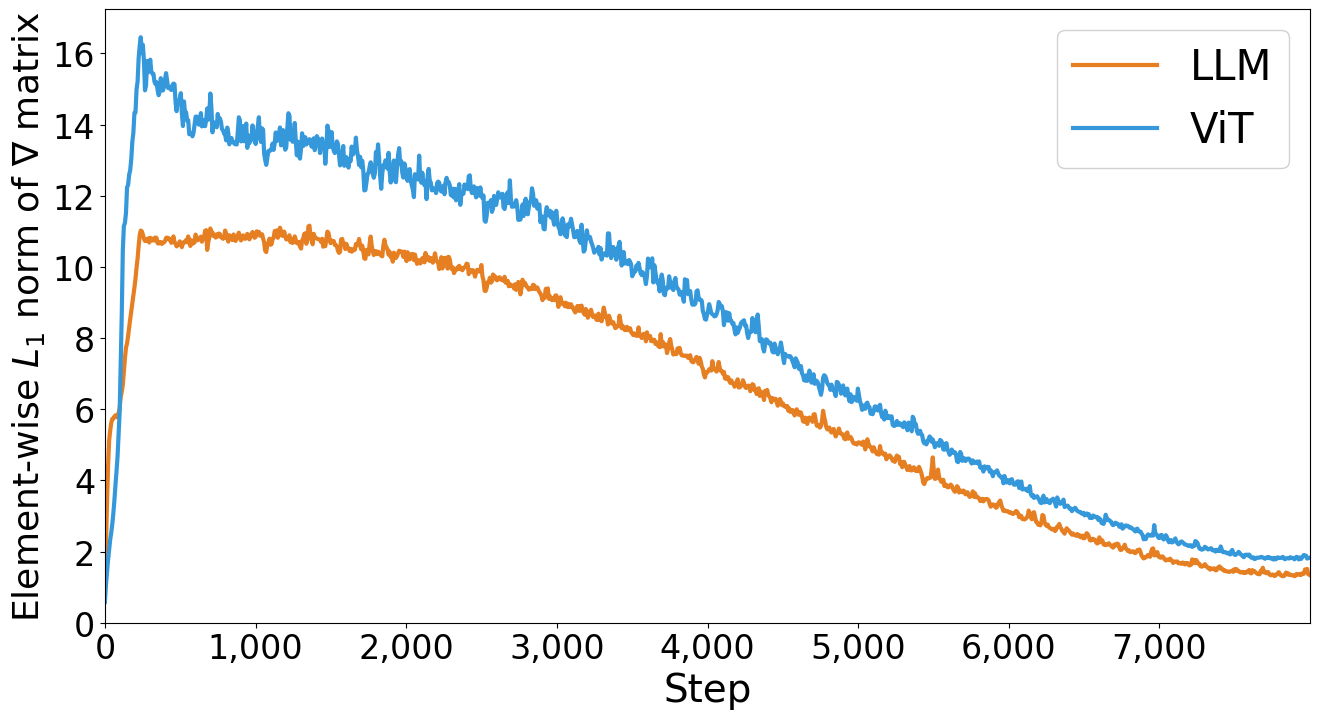}
    \caption{Gradient norm evolution during NanoVLM training. Element-wise L1 norms of weight gradients averaged across Vision layers matrices (blue) and attention projections (orange). ViT matrices consistently exhibit larger gradient magnitudes throughout training.}
    \label{fig:vit_component_gradient}
\end{subfigure}
\caption{Comparison of gradient norm evolution patterns across different model architectures during fine-tuning.}
\label{fig:gradient_comparison}
\end{figure}

\subsection{LLMs vs VLMs}
Perhaps a similarly interesting discovery we made is with the vision language model. Figure~\ref{fig:gradient_comparison} illustrates the training dynamics of the vision and language transformer components separately. Our analysis reveals an difference in convergence pattern like that between attention and MLP components: while the language transformer reaches convergence early in training, the vision transformer converges slower. However, extending training with a decreasing learning rate schedule leads to overfitting and subsequent accuracy degradation.

Our approach demonstrates significant accuracy improvements over full parameter fine-tuning. With only 8,000 training examples, we achieve 34.7\% accuracy—representing a 3.92\% absolute improvement over the baseline full fine-tuning approach. Remarkably, the original Nano-VLM requires approximately 17,000 examples to reach comparable performance levels\cite{wiedmann2025nanovlm}, indicating that our method reduces the required training data by more than 50\% while maintaining competitive accuracy.

\subsection{Combination with Parameter-Efficient Fine-Tuning Methods}
\textit{GradES} combines particularly well with parameter-efficient fine-tuning methods such as LoRA. As shown in Table~\ref{tab:timing-flops-comparison}, LoRA+GradES achieves the fastest training times across all models, reducing training time to just 0.14--0.38$\times$ of standard full-parameter fine-tuning. This dramatic speedup comes from optimizing two independent dimensions, (i) LoRA reduces the parameter space through low-rank weight decomposition, (ii) while \textit{GradES} reduces training iterations by detecting convergence through gradient monitoring. For example, on Qwen3 0.6B, LoRA+GradES completes training in just 907 seconds compared to 6,550 seconds for standard fine-tuning, a 7.22$\times$ speedup, while achieving better accuracy (67.30\% vs. 66.53\%).

The benefits of combining \textit{GradES} with LoRA become especially clear when compared to standard early stopping. While LoRA+ES actually increases training time by 2.97--6.90$\times$ due to validation overhead, LoRA+GradES maintains or improves LoRA's efficiency. This difference suggests that gradient-based early stopping criteria work particularly well in LoRA's constrained optimization landscape, where the reduced parameter space provides clearer convergence signals. Although LoRA methods inherently require more FLOPs per iteration than full-parameter training (2.07--2.43$\times$), \textit{GradES} compensates by reducing the total number of iterations needed. The result is a practical for deployment due to the extremely fast training times with reasonable computational costs, making LoRA+GradES the optimal choice for fine-tuning.

\subsection{GradES versus Classic Early Stopping}

 A fundamental limitation of classic early stopping is the computational expense of requiring complete forward passes for every validation step. In contrast, \textit{GradES} reuses gradient information from backpropagation, yielding substantial computational savings. As shown in Table~\ref{tab:timing-flops-comparison}, it achieves up to 6.79$\times$ speedup compared to classic early stopping on Qwen3-0.6B (907s vs 6,155s for LoRA fine-tuning), while maintaining comparable accuracy. The minimal accuracy difference---67.30\% for \textit{GradES} versus 67.37\% for classic early stopping as shown in Table~\ref{tab:LLM_accuracy_results}---does not justify the significant computational overhead. Across all five models tested, \textit{GradES} consistently reduces training time by 35--66\% while achieving a 45--71\% reduction in FLOPs compared to baseline full fine-tuning, making it particularly valuable for limited resource settings.

Furthermore, classic early stopping employs a model-wide convergence criterion that is not suitable for Transformer architectures. As demonstrated in Figure~\ref{fig:within_layer_gradient}, different weight matrices within transformer layers exhibit varying convergence rates. Model-wide early stopping fails to account for this heterogeneity, potentially leading to overfitting in some parameters while underfitting in others. \textit{GradES} addresses this limitation by enabling component-specific convergence criteria, allowing MLP and attention components to be trained independently until each reaches its optimal stopping point. This ensures all weight matrices converge according to appropriate criteria rather than being halted or unnecessarily extended based on global validation metrics. 

\section{Limitation}
\label{section:limitation}
While \textit{GradES} demonstrates substantial efficiency gains, some limitations exists. First, gradient monitoring incurs approximately 3\% computational overhead, though this is negligible compared to the 1.57--7.22$\times$ training time speed up achieved. Second, the convergence threshold $\tau$ requires manual tuning across different models and tasks, with full-parameter fine-tuning requiring larger thresholds than LoRA and model scale affecting optimal values. For VLMs, we observe that vision and language components convergence at a different rate, motivating component specific thresholds for optimal performance. Third, our current implementation employs static freezing without patience mechanisms common in traditional early stopping, potentially leading to premature convergence decisions. Additionally, while we validate \textit{GradES} on transformer architectures, its applicability to other neural architectures, such as convolutional networks, graph neural networks, and emerging architectures like state space models, remains unexplored.

\section{Conclusion and Future work}
\label{section:conclusion}
Several promising directions emerge for future work. Automatic threshold selection through gradient statistics or meta learning could eliminate manual tuning, while incorporating patience parameters would allow components to temporarily violate convergence criteria before freezing. Dynamic freezing and unfreezing mechanisms could adapt to task complexity and distribution shifts, particularly beneficial when combined with specific thresholds for MLP versus attention components as suggested by Figure~\ref{fig:all_component_gradient}. Integration with complementary efficiency techniques like mixed precision training, gradient checkpointing, and structured pruning could yield multiplicative speedups. Extending \textit{GradES} beyond transformers to vision models, graph networks, and hybrid architectures would broaden its impact. Most ambitiously, applying gradient stopping to pretraining could significantly reduce the massive computational costs of foundation model development, while theoretical analysis of convergence criteria would provide principled guidelines for threshold selection and accuracy guarantees.

The key insight underlying our approach is that different transformer components exhibit distinct convergence behaviors during fine-tuning. In VLMs particularly, we observe that language transformers converge earlier than vision transformers. By recognizing and exploiting these diverse convergence patterns, \textit{GradES} allocates computational resources more efficiently than uniform training strategies. The method's ability to eliminate costly validation passes while providing precise convergence control represents a practical advancement for deploying large language models with limited resources. As the scale of language models continues to grow, gradient optimization strategies like \textit{GradES} will become increasingly critical for making these powerful models accessible to the broader research community and enabling rapid experimentation in real-world applications.

\bibliographystyle{unsrt} 
\bibliography{references} 
\newpage 

\appendix
\section{Theoretical Analysis: Selection of Norm}

The choice of norm for gradient monitoring significantly impacts both computational efficiency and convergence detection reliability. For a gradient matrix $G \in \mathbb{R}^{m \times n}$, we consider four matrix norm candidates: element-wise $L_1$ ($\|G\|_{1,1} = \sum_{i,j} |g_{ij}|$), Frobenius ($\|G\|_F = \sqrt{\sum_{i,j} g_{ij}^2}$), spectral ($\|G\|_2 = \sigma_{\max}(G)$), and subordinate L$_\infty$ ($\|G\|_\infty = \max_i \sum_j |g_{ij}|$).

We select the element-wise $L_1$ norm based on computational efficiency and convergence properties. The $L_1$ norm requires only $O(mn)$ operations through element-wise summation, avoiding expensive computations such as square roots (Frobenius) or singular value decomposition (spectral, $O(mn\min(m,n))$). Furthermore, the $L_1$ norm provides a stronger convergence criterion through the following theorem:

\begin{theorem}
For any matrix $A \in \mathbb{R}^{m \times n}$, the elementwise $L_1$ norm provides an upper bound for commonly used matrix norms:
\begin{align}
\|A\|_2 &\leq \|A\|_{1,1} \label{eq:spectral_bound}\\
\|A\|_F &\leq \|A\|_{1,1} \label{eq:frobenius_bound}\\
\|A\|_\infty &\leq \|A\|_{1,1} \label{eq:infty_bound}\\
\|A\|_1 &\leq \|A\|_{1,1} \label{eq:$L_1$_bound}
\end{align}
where $\|A\|_1 = \max_j \sum_i |a_{ij}|$ denotes the subordinate $L_1$ norm (maximum column sum).
\end{theorem}

\begin{proof}
For \eqref{eq:spectral_bound}, the spectral norm satisfies $\|A\|_2 \leq \sqrt{\|A\|_1 \cdot \|A\|_\infty}$ by the well-known inequality for induced norms. Since $\|A\|_1 = \max_j \sum_i |a_{ij}| \leq \sum_{i,j} |a_{ij}| = \|A\|_{1,1}$ and similarly $\|A\|_\infty \leq \|A\|_{1,1}$, we have:
$$\|A\|_2 \leq \sqrt{\|A\|_1 \cdot \|A\|_\infty} \leq \sqrt{\|A\|_{1,1} \cdot \|A\|_{1,1}} = \|A\|_{1,1}$$

For \eqref{eq:frobenius_bound}, note that $a_{ij}^2 \leq |a_{ij}| \cdot \max_{k,l}|a_{kl}| \leq |a_{ij}| \cdot \|A\|_{1,1}$ for any element. Summing over all indices:
$$\|A\|_F^2 = \sum_{i,j} a_{ij}^2 \leq \sum_{i,j} |a_{ij}| \cdot \|A\|_{1,1} = \|A\|_{1,1}^2$$
Therefore, $\|A\|_F \leq \|A\|_{1,1}$.

For \eqref{eq:infty_bound}, the subordinate L$_\infty$ norm is the maximum row sum: $\|A\|_\infty = \max_{1 \leq i \leq m} \sum_{j=1}^n |a_{ij}|$. Since the maximum row sum cannot exceed the sum of all elements: $\|A\|_\infty \leq \sum_{i=1}^m \sum_{j=1}^n |a_{ij}| = \|A\|_{1,1}$.

For \eqref{eq:$L_1$_bound}, similarly, the subordinate $L_1$ norm is the maximum column sum: $\|A\|_1 = \max_{1 \leq j \leq n} \sum_{i=1}^m |a_{ij}| \leq \sum_{i=1}^m \sum_{j=1}^n |a_{ij}| = \|A\|_{1,1}$.
\end{proof}

These relationships establish that monitoring the $L_1$ norm provides a universal upper bound for convergence detection. When $\|G\|_{1,1} < \tau$, we guarantee that all other standard matrix norms are also bounded by $\tau$, ensuring robust convergence detection across multiple norm perspectives while maintaining linear computational complexity.

\section{Convergence Properties}
\label{section:convergence}

We provide theoretical guarantees for the convergence of Algorithm~\ref{alg:GradES}. Our analysis demonstrates that \textit{GradES} converges to a stationary point of the loss function while ensuring computational efficiency through adaptive parameter freezing.

\begin{theorem}[Convergence of GradES]
\label{thm:gades_convergence}
Consider Algorithm~\ref{alg:GradES} applied to a loss function $\mathcal{L}: \mathbb{R}^d \rightarrow \mathbb{R}$ that is $L$-smooth and lower bounded by $\mathcal{L}^*$. With threshold $\tau > 0$, the algorithm satisfies:
\begin{enumerate}
    \item The loss sequence $\{\mathcal{L}(t)\}_{t=1}^T$ is non-increasing after warm up
    \item All frozen parameters satisfy $\|\nabla_{\mathbf{W}} \mathcal{L}\|_{1,1} < \tau$ at convergence
    \item The algorithm terminates in finite time with $\min_{t \in [T]} \|\nabla \mathcal{L}(t)\| \leq \tau$
\end{enumerate}
\end{theorem}

\begin{proof}
\textbf{Part 1: Monotonic Loss Decrease.}
After the warmup period ($t > 0.05T$), the cosine schedule ensures $\eta_t$ is decreasing. For any active parameter $\mathbf{W} \notin \mathcal{F}$ at step $t$, the update rule yields:
\begin{align}
\mathcal{L}(\mathbf{W}_{t+1}) &\leq \mathcal{L}(\mathbf{W}_t) + \langle \nabla \mathcal{L}(\mathbf{W}_t), \mathbf{W}_{t+1} - \mathbf{W}_t \rangle + \frac{L}{2}\|\mathbf{W}_{t+1} - \mathbf{W}_t\|^2 \\
&= \mathcal{L}(\mathbf{W}_t) - \eta_t \|\nabla \mathcal{L}(\mathbf{W}_t)\|^2 + \frac{L\eta_t^2}{2}\|\nabla \mathcal{L}(\mathbf{W}_t)\|^2
\end{align}

For frozen parameters $\mathbf{W} \in \mathcal{F}$, we have $\mathbf{W}_{t+1} = \mathbf{W}_t$, thus $\mathcal{L}(\mathbf{W}_{t+1}) = \mathcal{L}(\mathbf{W}_t)$. The cosine schedule with maximum value $\eta_0 < \frac{2}{L}$ ensures:
\begin{equation}
\mathcal{L}(t+1) \leq \mathcal{L}(t) - \sum_{\mathbf{W} \notin \mathcal{F}} \eta_t \left(1 - \frac{L\eta_t}{2}\right) \|\nabla_{\mathbf{W}} \mathcal{L}(t)\|^2
\end{equation}

Since $\eta_t \leq \eta_0 < \frac{2}{L}$ throughout training, the loss sequence is non-increasing.

\textbf{Part 2: Frozen Parameters at Stationary Points.}
A parameter matrix $\mathbf{W}$ is frozen at step $t_f$ when $\|\nabla_{\mathbf{W}} \mathcal{L}(t_f)\|_{1,1} < \tau$. Since frozen parameters receive no further updates:
\begin{equation}
\mathbf{W}_{t} = \mathbf{W}_{t_f} \quad \forall t > t_f
\end{equation}

As shown in Part 1, the gradient magnitudes are non-increasing after warmup. Combined with the continuity of gradients:
\begin{equation}
\|\nabla_{\mathbf{W}} \mathcal{L}(t)\|_{1,1} \leq \|\nabla_{\mathbf{W}} \mathcal{L}(t_f)\|_{1,1} < \tau \quad \forall t > t_f > 0.05T
\end{equation}

This ensures that once a parameter is frozen, its gradient remains below the threshold.

\textbf{Part 3: Finite-Time Termination.}
Define the active parameter set at time $t$ as $\mathcal{A}_t = \{\mathbf{W} : \mathbf{W} \notin \mathcal{F}_t\}$. The cardinality $|\mathcal{A}_t|$ is non-increasing since parameters can only transition from active to frozen. 

Under the cosine schedule, as $t \rightarrow T$, we have $\eta_t \rightarrow 0$. In experiment, gradient magnitudes decrease monotonically after warmup. Therefore, there exists a finite $T^* < T$ such that either:
\begin{itemize}
    \item All parameters satisfy $\|\nabla_{\mathbf{W}} \mathcal{L}\|_{1,1} < \tau$ and are frozen, or
    \item The cosine schedule drives $\eta_t \|\nabla \mathcal{L}(t)\| < \epsilon$ for arbitrarily small $\epsilon$
\end{itemize}

In both cases, the algorithm effectively converges with $\min_{t \in [T]} \|\nabla \mathcal{L}(t)\| \leq \tau$.
\end{proof}

\begin{corollary}
\textit{GradES} achieves an $\epsilon$-stationary point while potentially reducing computational cost by a factor proportional to $|\mathcal{F}|/d$, where $d$ is the total number of parameters. The cosine schedule ensures smooth convergence without oscillations in gradient magnitudes.
\end{corollary}

This analysis establishes that \textit{GradES} maintains convergence guarantees under the practical cosine learning rate schedule used in our experiments. The threshold $\tau$ controls the trade-off between convergence accuracy and computational savings, while the 5\% warmup period ensures stable gradient behavior before monitoring begins.

\section{Hyperparameter Configuration}
\label{section:hyperparameter}
We present comprehensive hyperparameter configurations to ensure experimental reproducibility. Tables~\ref{tab:basic-hyperparameters}--\ref{tab:grades-hyperparameters} detail the training configurations for five language models—Qwen3-14B, Phi4-14B, Qwen3-0.6B, Llama-3.1-8B, and Mistral-7B—under both full-parameter (FP) fine-tuning and Low-Rank Adaptation (LoRA). For all experiments, we employ early stopping with a validation loss threshold of $\delta = 0.0005$, patience of 3 epochs, and validation performed at 5\% intervals throughout training. Tables~\ref{tab:vlm-hyperparameters}--\ref{tab:vlm-efficiency} additionally present configurations and efficiency metrics for vision-language model experiments.

\begin{table}[H]
\centering
\caption{Hyperparameter configuration for language model fine-tuning experiments. FP denotes full-parameter fine-tuning.}
\label{tab:basic-hyperparameters}
\begin{tabular}{@{}lllcccc@{}}
\toprule
\textbf{Model} & \textbf{Method} & \textbf{Learning Rate} & \textbf{Batch Size} & \textbf{Grad Accum} & \textbf{Max Seq Len} & \textbf{LoRA Rank} \\
\midrule
\multirow{2}{*}{Qwen3 14B} 
& FP         & 2e-5  & 1  & 4 & 4096 & - \\
& LoRA       & 2e-4  & 16 & 4 & 4096 & 32 \\
\midrule
\multirow{2}{*}{Phi4 14B}
& FP         & 2e-5 & 1  & 4 & 4096 & - \\
& LoRA       & 2e-4   & 16 & 4 & 4096 & 32 \\
\midrule
\multirow{2}{*}{Qwen3 0.6B}
& FP         & 2e-5  & 1 & 4 & 4096 & - \\
& LoRA       & 2e-4  & 16 & 4 & 4096 & 32 \\
\midrule
\multirow{2}{*}{Llama-3.1-8B}
& FP         & 2e-5  & 1  & 4 & 4096 & - \\
& LoRA       & 2e-4  & 16 & 4 & 4096 & 32 \\
\midrule
\multirow{2}{*}{Mistral-7B}
& FP         & 2e-5  & 1  & 4 & 4096 & - \\
& LoRA       & 2e-4  & 16 & 4 & 4096 & 32 \\
\bottomrule
\end{tabular}
\end{table}

\begin{table}[H]
\centering
\caption{\textit{GradES} convergence thresholds for language models. Grace period ratio ($\alpha$) determines the fraction of training steps before gradient monitoring begins. Threshold $\tau$ defines the gradient magnitude below which components are frozen. Full-parameter fine-tuning requires higher thresholds due to larger gradient magnitudes.}
\label{tab:grades-hyperparameters}
\begin{tabular}{@{}llcc@{}}
\toprule
\textbf{Model} & \textbf{Method} & \textbf{Grace Period Ratio($\alpha$)} & \textbf{Threshold Tau ($\tau$)} \\
\midrule
\multirow{2}{*}{Qwen3 14B} 
& FP   & 0.55 & 6.387926  \\
& LoRA & 0.55& 0.025181  \\
\midrule
\multirow{2}{*}{Phi4 14B}
& FP   & 0.55 & 3.512882\\
& LoRA & 0.55 & 0.025181 \\
\midrule
\multirow{2}{*}{Qwen3 0.6B}
& FP   & 0.55 & 1.804456 \\
& LoRA & 0.55 & 0.001183 \\
\midrule
\multirow{2}{*}{Llama-3.1-8B}
& FP   & 0.55 & 2.404167  \\
& LoRA & 0.55 & 0.021637  \\
\midrule
\multirow{2}{*}{Mistral-7B}
& FP   & 0.55 & 2.726866  \\
& LoRA & 0.55 & 0.029591 \\
\bottomrule
\end{tabular}
\end{table}

\begin{table}[H]
\centering
\caption{Hyperparameter configuration for vision-language model experiments. Component-specific thresholds enable targeted convergence monitoring for vision and language transformers separately.}
\label{tab:vlm-hyperparameters}
\begin{tabular}{@{}llccc@{}}
\toprule
\textbf{Model} & \textbf{Method} & \textbf{Vision $\tau$} & \textbf{Language $\tau$} & \textbf{Grace Period ($\alpha$)} \\
\midrule
\multirow{2}{*}{Qwen2.5-VL-7B}
& LoRA & 3.3 & 33.0 & 0.30 \\
& FP (4-bit) & 0.13 & 0.09 & 0.30 \\
\midrule
\multirow{2}{*}{NanoVLM}
& Training & 0.30 & 6.00 & 0.28\\
& Training (alt) & 6.30 & 3.90 & 0.28 \\
\bottomrule
\end{tabular}
\end{table}

\begin{table}[H]
\centering
\caption{Computational efficiency of \textit{GradES} on vision-language models. FLOPs and training time measurements demonstrate consistent improvements across different fine-tuning methods and precision settings.}
\label{tab:vlm-efficiency}
\begin{tabular}{@{}llccc@{}}
\toprule
\textbf{Model} & \textbf{Method} & \textbf{FLOPs} & \textbf{Time (s)} & \textbf{Speedup} \\
\midrule
\multirow{2}{*}{Qwen2.5-VL-7B}
& LoRA (bf16) & 2.80E+19 & 103,466.02 & - \\
& LoRA+GradES (bf16) & 2.52E+19 & 87,572.36 & 1.18$\times$ \\
& FP (4-bit) & 2.27E+19 & 157,239.03 & - \\
& FP+GradES (4-bit) & 1.99E+19 & 134,686.05 & 1.17$\times$ \\
\midrule
\multirow{2}{*}{NanoVLM}
& Training & - & 62,009.00 & - \\
& Training+GradES & - & 34,669.00 &1.79$\times$ \\
\bottomrule
\end{tabular}
\end{table}

\section{Code Availability}
\label{section:code_availability}

We are committed to ensuring the reproducibility of our research. To facilitate this, we provide comprehensive resources:

\noindent\textbf{Implementation.} Our complete implementation, including training scripts, evaluation pipelines, and gradient monitoring utilities, is publicly available at \url{https://github.com/IXZZZ9/GradES}. The repository includes detailed documentation, environment setup instructions, and scripts to reproduce all experimental results presented in this paper.

\noindent\textbf{Licensing.} All code is released under the MIT License, promoting open scientific collaboration and industrial adoption. Model weights follow the original Qwen license terms.

\end{document}